\documentclass[letterpaper, 10 pt, conference]{ieeeconf}  

\IEEEoverridecommandlockouts                              

\usepackage{amsmath,amsfonts}
\usepackage{array}
\usepackage{textcomp}
\usepackage{stfloats}
\usepackage{url}
\usepackage{verbatim}
\usepackage{graphicx}
\usepackage{cite}
\usepackage{bm}
\usepackage{array}
\usepackage{amssymb}
\makeatletter
\let\NAT@parse\undefined
\makeatother
\usepackage{hyperref}  
\hypersetup{colorlinks=true,
            linkcolor=blue,
            anchorcolor=blue,
            citecolor=blue
            }
\usepackage{float}
\usepackage{subfigure}
\usepackage{color}
\usepackage{multirow}
\newtheorem{rem}{Remark}

\newtheorem{theorem}{Theorem}
\usepackage {amssymb}
\usepackage{algpseudocode}
\usepackage{footmisc}
\usepackage{threeparttable}
\usepackage{booktabs}
\usepackage[ruled,vlined,linesnumbered]{algorithm2e}

\algnewcommand\algorithmicforeach{\textbf{for each}}
\algdef{S}[FOR]{ForEach}[1]{\algorithmicforeach\ #1\ \algorithmicdo}

\title{\LARGE \bf
IDF-MFL: Infrastructure-free and Drift-free Magnetic Field Localization for Mobile Robot
}

\author{Hongming Shen,~\IEEEmembership{Member,~IEEE}, Zhenyu Wu, Wei Wang,~\IEEEmembership{Graduate Student Member,~IEEE}, \\Qiyang Lyu, Huiqin Zhou, and Danwei Wang,~\IEEEmembership{Fellow,~IEEE}
\thanks{This research is supported by the National Research Foundation, Singapore, under the NRF Medium Sized Centre scheme (CARTIN). Any opinions, findings and conclusions or recommendations expressed in this material are those of the author(s) and do not reflect the views of National Research Foundation, Singapore and Maritime and Port Authority of Singapore.}
\thanks{All authors are with the School of Electrical and Electronic Engineering,
Nanyang Technological University, Singapore 639798. \textit{Corresponding author: Hongming Shen (hongming.shen@ntu.edu.sg)}}%
}

\begin{document}

\maketitle
\thispagestyle{empty}
\pagestyle{empty}

\begin{abstract}
In recent years, infrastructure-based localization methods have achieved significant progress thanks to their reliable and drift-free localization capability.
However, the pre-installed infrastructures suffer from inflexibilities and high maintenance costs.
This poses an interesting problem of how to develop a drift-free localization system without using the pre-installed infrastructures.
In this paper, an infrastructure-free and drift-free localization system is proposed using the ambient magnetic field (MF) information, namely IDF-MFL.
IDF-MFL is infrastructure-free thanks to the high distinctiveness of the ambient MF information produced by inherent ferromagnetic objects in the environment, such as steel and reinforced concrete structures of buildings, and underground pipelines.
The MF-based localization problem is defined as a stochastic optimization problem with the consideration of the non-Gaussian heavy-tailed noise introduced by MF measurement outliers (caused by dynamic ferromagnetic objects), and an outlier-robust state estimation algorithm is derived to find the optimal distribution of robot state that makes the expectation of MF matching cost achieves its lower bound.
The proposed method is evaluated in multiple scenarios\footnote{\url{https://youtu.be/pTRD7SZRuA8}}, including experiments on high-fidelity simulation, and real-world environments. The results demonstrate that the proposed method can achieve high-accuracy, reliable, and real-time localization without any pre-installed infrastructures.
\end{abstract}

\section{Introduction}
Localization is a fundamental task in developing autonomous robotic systems which has the potential to enable extensive industrial applications, such as construction\cite{doi:10.1126/scirobotics.abp9758}, logistics\cite{doi:10.1126/scirobotics.adm7020}, and underground mine exploration\cite{doi:10.1126/scirobotics.abp9742}.
All of these missions rely on reliable state estimation for mobile robots.
During the last decade, simultaneous localization and mapping (SLAM) technology\cite{FASTLIO2,ZHU2021107185,vins-mono} has been well applied to mobile robots.
However, localization in enclosed or partially enclosed environments (e.g. corridors, industrial warehouses, carparks) still remains challenging and intractable due to the inevitable drift of the SLAM system.
Even though the loop closure\cite{6202705,8593953} and bundle adjustment (BA) \cite{8793749,deng2023plgslam,10263983} techniques are able to correct drift, the state can change drastically when the drift is eliminated.
The state transition is unacceptable for the mobile robots due to the causing of destabilization.
To overcome the aforementioned limitations of SLAM systems, pre-installed infrastructures, such as ultra-wideband (UWB) anchor\cite{9502143}, radio frequency identifications (RFIDs)\cite{9509296}, and QR codes\cite{8602360}, are typically deployed in industrial scenarios.
However, the installation and maintenance process of infrastructures-based localization systems is inflexible and costly which significantly limits the industrial application of the mobile robot system.
Hence, inspired by natural animals, such as spiny birds and lobsters, that can sense their position and orientation using information from the local anomalies of earth MF \cite{Xu2021,Boles2003}, an infrastructure-free localization system is investigated for mobile robots by using the local ambient MF.
\subsection{Related Works}
In recent years, infrastructure-free localization systems have been amenable for wide-scale commercial use.
A serial of priori-map-based SLAM methods is investigated to realize smooth and drift-free localization without any infrastructures.
In \cite{https://doi.org/10.1002/rob.21936}, a real‐time high‐precision visual localization system is designed for autonomous vehicles that employ only low‐cost stereo cameras to localize the vehicle with a priori 3D LiDAR map.
The drift of the visual odometry is eliminated by registering the visual point cloud to the pre-build LiDAR map through a probabilistic weighted Normal Distributions Transformation (NDT) \cite{doi:10.1177/0278364912460895}.
Considering that visual feature detection and matching are unstable in texture-less or repetitive scenarios, a semantic pre-build map is adopted in \cite{9340939} which contains typical features in parking lots, such as guide signs, parking lines, speed bumps, etc.
However, the priori-map-based SLAM methods typically degrade significantly or even become unobservable in conditions of the visual/LiDAR odometry unreliable caused by poor illumination or self-similar scenes.

To overcome this, an ambient MF-based localization system has become as a viable alternative for infrastructure-free localization thanks to the distinctiveness and pervasiveness of MF distortions caused by ferromagnetic objects.
The use of MF as a source of SLAM is a promising novel approach.
To achieve MF-based SLAM, the MagSLAM \cite{6817910} employs a grid-based spatial discretization methodology and assumes the MF intensity in one grid with respect to the same distribution.
In \cite{8455789}, a pedestrian dead reckoning (PDR)-aided MF SLAM is investigated which represents the MF map with a reduced-rank Gaussian process\cite{Solin2020} using the Laplace basis functions, and a Rao-Blackwellized particle filter (PF) is adopted to compensate for position drift in PDR.
To further improve the real-time performance of \cite{8455789}, the PF is replaced by extended Kalman filter (EKF) in \cite{s22082833} thanks to the derivation of the MF gradient expressions.
However, similar to the visual/LiDAR SLAM \cite{vins-mono,FASTLIO2}, the MF SLAM also suffers from drift when the pre-build MF map is unavailable.

A series of works investigate the process of offline MF map construction through the bilinear interpolation \cite{HAVERINEN20091028}, Gaussian process \cite{8373720}, and learning-based MF prediction \cite{POLLOK2023170556}.
In \cite{MFGN}, a drift-free MF-based localization algorithm is proposed by fusing the MF-based localization with PDR through the EKF, where the MF-based localization algorithm is designed to estimate the rigid transformation that optimally aligns the MF measurements with the pre-build MF map through the Gauss-Newton optimization.
Considering the potential MF gradient flatness region in real-world environments, which can lead to the degeneracy of gradient-based state estimation (e.g. Gauss-Newton, Levenberg-Marquardt, and Kalman filter) \cite{7487211}, PF \cite{6696459,7354010,MFPF} is well-applied thanks to its excellent performance in non-convex optimization.
In \cite{9947016}, a complete MF localization system is designed by integrating offline MF map construction, PDR odometry, and a dynamic time warping (DTW)-based MF sequence matching \cite{10.1145/2508037.2508054} into an indoor pedestrian localization framework.
Both \cite{MFGN,6696459,7354010,MFPF,9947016} are wheel odometry/PDR-aided MF localization systems, which limited them can only deploy on certain robot platforms (legged robots or robots equipped with the wheel encoder).
To overcome this problem, a few works \cite{9095809,8683984,WUICRA23} paid attention to the pure MF localization algorithm.
In \cite{9095809} and \cite{8683984} the MF-based localization is realized through the classical K-Nearest Neighbors (KNN) algorithm.
However, the KNN-based MF matching is typically limited by the similar-sequential-route assumption (the online localization route should be similar to the route used in the offline MF map construction process) due to the 3D MF vector measured by a magnetometer at one single location is highly dependent on the orientation.
In \cite{WUICRA23}, a pure MF localization method is proposed with the maximum a posteriori (MAP) estimator which solves the similar-sequential-route limitation problem with a rotation-invariant MF descriptor.
\subsection{Motivation and Contributions}
In view of the aforementioned analysis, despite the recent popularity of ambient MF-based state estimation research for infrastructure-free localization, most of the works are drift-suffered\cite{6817910,8455789,s22082833}, auxiliary odometry-aided\cite{MFGN,6696459,7354010,MFPF,9947016}, or similar-sequential-route limited\cite{9095809,8683984}.
However, drift is unacceptable for mobile robots in enclosed or partially enclosed environments, and the dependence on auxiliary odometry limits the application of the algorithm into certain robot platforms, while the similar-sequential-route limitation leads the robot to track with the fixed path.
With the goal of designing a flexible drift-free localization system, in this work, a robust pure MF-based localization system is developed without any assistance from infrastructures.
The key idea of the proposed infrastructures-free, drift-free localization system is to find the optimal distribution of state that makes the expectation of MF matching cost achieves its lower bound.
The MF matching cost is designed to describe the difference between the real-time MF measurements and the pre-built MF map which integrates correct MF matching measurement and MF measurement outlier (ignored in traditional MF-based localization methods) simultaneously in a piecewise function.
Different from the traditional MF-based localization methods\cite{6817910,8455789,s22082833,MFGN,6696459,7354010,MFPF,9095809,8683984,WUICRA23}, which generally approximate the MAP problem with the Gaussian distribution assumption, the proposed method deals with non-Gaussian heavy-tailed noise (MF measurement outlier introduced) robustly by optimizing the expectation of cost function with a stochastic optimization algorithm. The main contributions of this paper are listed as follows:
\begin{itemize}
    \item A pure MF-based localization approach is proposed for mobile robots which is infrastructure-free, drift-free, and does not rely on any additional odometry information.
    \item To deal with the non-Gaussian heavy-tailed noise, an outlier-robust state estimator is derived which can optimize the non-convex and noncontinuous state estimation problem with parallel sampling.
    \item The practical implementations of the proposed stochastic estimator are illustrated in detail, including Monte-Carlo approximation, sampling dimension shrinking on manifolds, and cost-shifting strategy. The practicability and performance of the proposed MF-based localization system are extensively verified in both simulated and real-world experiments.
\end{itemize}
\section{System Overview}
The overall design of the proposed MF-based localization system is given in \mbox{Fig. \ref{system overview}}, which can be divided into an offline MF map construction phase and an online MF-based localization phase.
For offline MF map construction, an external odometry is introduced to project the MF measurements into the world frame.
Although the loop closure and BA of the SLAM system can lead to state transition in the online localization phase, these techniques are very suitable to provide a drift-free and high-accuracy trajectory for the offline mapping phase.
In this work, a LiDAR SLAM algorithm \cite{FASTLIO2} is adopted to provide drift-free pose estimation with loop closure \cite{8593953}, and a global LiDAR BA \cite{10263983} is also implemented to further improve the accuracy of the estimated offline trajectory.
After projecting the collected MF data into the world frame, a MF interpolation is adopted using the Gaussian process \cite{8373720} to produce a dense MF map.
The MF map is stored using the hash data structure in a certain resolution (0.05m used in this work) to achieve an effective data insertion and index operation in a constant time complexity of $O(1)$.
For the online localization phase, the magnetic matching cost is constructed by measuring the difference between the online collected MF measurements and their correspondences in the pre-build MF map (indexed in the hash table).
The online MF-based localization is realized by minimizing the expectation of the MF matching cost in a stochastic optimization manner (\mbox{Section \ref{sec: Robust Estimation via Stochastic Optimization}}), 
and a Monte-Carlo sampling is adopted to approximate the optimal distribution of state which makes the expectation of cost function achieve its lower bound (\mbox{Section \ref{sec: Practical Implementation for Magnetic Field-based State Estimation}}).
\begin{rem} \label{rem: pure MF}
    Different from auxiliary odometry-aided MF-based localization methods \cite{MFGN,6696459,7354010,MFPF,9947016}, in this work, the external odometry information is only used in the offline MF map construction 
    phase.
    In industrial applications, map construction devices and mobile robots are typically two separate entities.
    Hence, during the online localization phase, low-cost pure MF localization can be achieved without LiDAR or platform-specific odometry systems, such as PDR, wheel odometry.
\end{rem}

\begin{figure}[!t]\centering
\includegraphics[width=\linewidth]{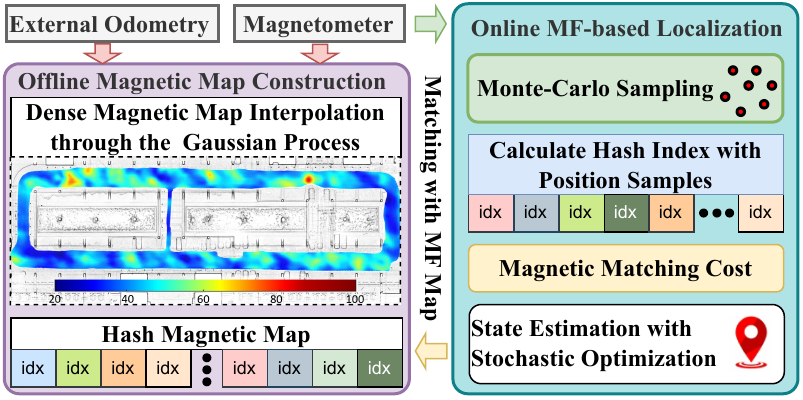}
\vspace{-2em}
\caption{System overview of IDF-MFL. As noted in \textit{Remark \ref{rem: pure MF}}, the external odometry is only used in offline MF map construction phase.}\label{system overview}
\vspace{-1em}
\end{figure}

\section{Problem formulation} \label{sec: Problem formulation}
The objective of the proposed localization system is to utilize magnetic sensing to determine the robot's state, producing robust, drift-free, and infrastructure-free estimates in repetitive environments.
Given a magnetic map $\mathbb M$ and a sequence of magnetic measurements ${\mathbb B}_{t_k} = \{\mathbf{B}^i_{t_k}\}_{i=1}^N$, with ${\mathbf{B}^i_{t_k}} = [B^i_{x,t_k}, B^i_{y,t_k}, B^i_{z,t_k}]^\top$ the $i$-th magnetic sensor measurements within the magnetic sensor frame.
The MF-based localization measurements can be represented as the following model:
\begin{equation}
\begin{aligned}
    &{{\bf{R}}^i_{{t_k}}}{{\bf{B}}^i_{{t_k}}} = \mathbb{M}(\mathbf{p}_{t_k}^i)+ {{\bf{o}}^i_{{t_k}}} + {{\bf{n}}_{{t_k}}}, i = 1,\ldots, N\\
    &{\bf{R}}_{{t_k}}^i = {{\bf{R}}_{{t_k}}}{}^b{\bf{R}}_m^i,{\bf{p}}_{{t_k}}^i = {{\bf{R}}_{{t_k}}}{}^b{\bf{p}}_m^i + {{\bf{p}}_{{t_k}}}
\end{aligned}    
\end{equation}
where $\mathbf{R}_{t_k} \in \text{SO3}$ and $\mathbf{p}_{t_k} \in \mathbb{R}^3$ are the unknown (to-be-estimated) rotation and translation vector, $(\cdot)_{t_k}$ indicates the variable obtained at time $t_k$, $N$ is the number of magnetic sensors installed on the robot, $^{b}{\bf R}^i_{m}$ and $^{b}{\bf p}^i_{m}$ are extrinsic between $i$-th magnetic sensor frame and body frame, ${\mathbb{M}}({{\bf{p}}^i_{{t_k}}})$ is a function which returns the ambient MF of pre-build MF map $\mathbb M$ at position ${{\bf{p}}^i_{{t_k}}}$ within the local frame, ${\bf o}^i_{t_k}$ is a zero vector if the measurement is an inlier, or a vector of arbitrary number for outlier measurement, and ${\bf n}_{t_k}$ models the measurement noise.

Traditional MF-based state estimation methods \cite{WUICRA23,MFGN,MFPF} formulate the localization problem as a MAP problem which is solved through the iterative least-squares optimization or PF with the Gaussian approximation.
However, the estimation performance of both least-squares optimization and PF degrades dramatically under non-Gaussian heavy-tailed noises which are often induced by measurement outliers $\mathbf{o}_{t_k}^i$.
Hence, the present work defined the MF-based state estimation problem as a stochastic optimization problem\cite{schneider2007stochastic}:
\begin{equation}
{{{\bf{\hat x}}}_{{t_k}}} = \mathop {\arg \min }\limits_{{{\mathbf{x}}_{{t_k}}} \sim \mathcal{Q}} \mathbb{E}_{\mathcal{Q}}\left[ {S\left( {{{\mathbf{x}}_{{t_k}}}} \right)} \right]
\label{Eq: stochastic opt}
\end{equation}
where $\mathcal{Q}$ is treated as the basis distribution of the state vector $\mathbf{x}_{t_k}$, $\mathbb{E}_{\mathcal{Q}}[\cdot]$ denotes the expectation operation over the state vector $\mathbf{x}_{t_k}$ with respect to $\mathcal{Q}$, $S(\mathbf{x}_{t_k})$ is the cost function defined in (\ref{Eq: TLS}), and the optimal state vector $\mathbf{x}^*_{t_k}$ is defined as:
\begin{equation}
    \mathbf{x}^*_{t_k} = \left[\mathbf{p}_{t_k}^{*\top}, {\bm {\phi}}_{t_k}^{*\top}, \mathbf{v}_{t_k}^{*\top}, \mathbf{\bm {\omega}}_{t_k}^{*\top}\right]^\top
\end{equation}
where ${\bm\phi}_{t_k}\in {\mathbb R}^3$ is the axis-angle representation of the rotation matrix $\mathbf{R}_{t_k} \in \text{SO3}$,
$\mathbf{v}_{t_k}$ and ${\bm \omega}_{t_k}$ are linear velocity and angular velocity, respectively.

Assume the noise of the MF-based measurements is unknown but bounded \cite{milanese1989estimation}, the cost function $S(\mathbf{x}_{t_k})$ can be written as a piecewise function according to the definition of $\mathbf{o}_{t_k}^i$. 
\begin{equation}
S\left( {{{\bf{x}}_{{t_k}}}} \right) = \sum\limits_{i = 1}^N {\left\{ {\begin{array}{*{20}{l}}
{{{\left\| {{\mathbb{M}}({\bf{p}}_{{t_k}}^i) - {\bf{R}}_{{t_k}}^i{\bf{B}}_{{t_k}}^i} \right\|}^2}}&{,inlier}\\
c^2&{,outlier}
\end{array}} \right.} \label{Eq: TLS}
\end{equation}
where $c^2$ is a value larger than the upper bound of ${{\left\| {{\mathbb{M}}({\bf{p}}_{{t_k}}^i) - {\bf{R}}_{{t_k}}^i{\bf{B}}_{{t_k}}^i} \right\|}^2}$.
\section{Robust Estimation via Stochastic Optimization} \label{sec: Robust Estimation via Stochastic Optimization}
Solving the stochastic optimization problem (\ref{Eq: stochastic opt}) is to find the optimal state vector $\mathbf{x}^*_{t_k}$ that minimizes the expectation of the cost function $S(\mathbf{x}_{t_k})$.
Assume that the optimal distribution $\mathcal{Q}^*$ can be defined, which makes $\mathbb{E}_{\mathcal{Q}^*}$ provides the lower bound of $\mathbb{E}_{\mathcal{Q}}$. The stochastic optimization problem (\ref{Eq: stochastic opt}) can be solved by calculating the expectation of $\mathbf{x}_{t_k}$ over the optimal distribution $\mathcal{Q}^*$.
\begin{equation}
\begin{aligned}
{{{\bf{x}}}^*_{{t_k}}} &= \mathop {\arg \min }\limits_{{{\bf{x}}_{{t_k}}} \sim {\cal Q}} {\mathbb{E}_{\cal Q}}\left[ {S\left( {{{\bf{x}}_{{t_k}}}} \right)} \right]\\
 &\buildrel \Delta \over = {{\mathbb{E}}_{{\mathcal{Q}^*}}}\left[ {{{\bf{x}}_{{t_k}}}} \right] = \int{q^* \left( {\bf{x}} \right){\bf{x}}d{\bf{x}}}
\end{aligned} \label{Eq: simple stochastic opt}
\end{equation}
where $q^*(\cdot)$ is the probability density function of the optimal distribution $\mathcal{Q}^*$. As a consequence of (\ref{Eq: simple stochastic opt}), if the optimal probability density function $q^*$ can be defined, the stochastic optimization problem can solved by directly sampling from the optimal distribution.
Hence, the definition of the optimal probability density function is crucial for solving (\ref{Eq: simple stochastic opt}), which is derived from the lower bound of $\mathbb{E}_{\mathcal{Q}^*}\left[S\left(\mathbf{x}_{t_k}\right)\right]$ (given in \textit{Theorem \ref{Them: 1}}).
\begin{theorem} \label{Them: 1}
The lower bound of $\mathbb{E}_{\mathcal{Q}^*}\left[S\left(\mathbf{x}_{t_k}\right)\right]$ can be defined as:
\begin{equation}
\begin{aligned}
&{\mathbb{E}_{{{\cal Q}^*}}}\left[ {S\left( {{{\bf{x}}_{{t_k}}}} \right)} \right] \ge \\  &- \lambda \log \left({\mathbb{E}_{\cal Q}}\left[ {\exp \left( { - \frac{1}{\lambda }S\left( {{{\bf{x}}_{{t_k}}}} \right)} \right)} \right]\right) - \lambda {\mathbb{D}}\left( {{{\cal Q}^*}||{\cal Q}} \right) \label{Eq: theorem 1}
\end{aligned}
\end{equation}
where $\lambda \in \mathbb{R}^+$ and ${\mathbb{D}}\left( {{{\cal Q}^*}||{\cal Q}} \right)$ denotes the KL divergence between distributions ${\cal Q}^*$ and ${\cal Q}$.
\end{theorem}
\begin{proof}
For (\ref{Eq: theorem 1}), the expectation of the first term on the right-hand side can be switched by using the standard importance sampling technique\cite{doucet2001sequential}.
\begin{equation}
\begin{aligned}
- \lambda \log &\left({\mathbb E}{_\mathcal Q}\left[\exp \left( - \frac{1}{\lambda }{{S}}\left({\mathbf{x}_{t_k}}\right)\right)\right]\right) \\&=
- \lambda \log \left({\mathbb E}{_{\mathcal{Q}^*}}\left[\exp \left( - \frac{1}{\lambda }{{S}}\left({\mathbf{x}_{t_k}}\right)\right)\frac{q}{q^*}\right]\right) 
\end{aligned} \label{Eq: proof importance sampling}
\end{equation}
where $q$ is the density function correspondence to distribution $\cal Q$.
Using Jensen's inequality and the concavity of the logarithm, the upper bound of the right-hand side of (\ref{Eq: proof importance sampling}) can be defined as:
\begin{equation}
\begin{aligned}
- \lambda \log &\left({\mathbb E}{_{\mathcal{Q}^*}}\left[\exp \left( - \frac{1}{\lambda }{{S}}\left({\mathbf{x}_{t_k}}\right)\right)\frac{q}{q^*}\right]\right) \le \\&- \lambda \mathbb{E}_{\mathcal{Q}^*}\left[\log\left(\exp\left(-\frac{1}{\lambda}{S}(\mathbf{x}_{t_k})\right)\frac{q}{q^*}\right)\right]
\end{aligned} \label{Eq: proof jensen}
\end{equation}
The right-hand side of (\ref{Eq: proof jensen}) can be simplified using the definition of KL-divergence. 
\begin{equation}
\begin{aligned}
- \lambda \mathbb{E}_{\mathcal{Q}^*}\left[\log\left(\exp\left(-\frac{1}{\lambda}{S}(\mathbf{x}_{t_k})\right)\frac{q}{q^*}\right)\right]& \\= \mathbb{E}_{\mathcal{Q}^*}\left[{S}(\mathbf{x}_{t_k})\right] &+ \lambda{\mathbb D}({\mathcal Q}^*||\mathcal Q)
\end{aligned} \label{Eq: KL}
\end{equation}
Substituting (\ref{Eq: proof jensen}) and (\ref{Eq: KL}) into (\ref{Eq: proof importance sampling}) yields the inequation (\ref{Eq: theorem 1}). This completes the proof.
\end{proof}

According to \textit{Theorem \ref{Them: 1}}, the optimal probability density function can be derived.
Expanding the KL-divergence ${\mathbb{D}}\left( {{{\cal Q}^*}||{\cal Q}} \right)$, (\ref{Eq: theorem 1}) can be rewritten as:
\begin{equation}
\begin{aligned}
{{\mathbb{E}}_{{{\cal Q}^*}}}&\left[ {\log \left( {\frac{{{q^*}}}{q}} \right)} \right] = {{\mathbb{E}}_{{{\cal Q}^*}}}\left[ {\log \left( {{{{q^*}}}} \right)} \right] -{{\mathbb{E}}_{{{\cal Q}^*}}}\left[ {\log \left( {{{{q}}}} \right)} \right] \ge\\&- \log\left({{\mathbb{E}}_{\cal Q}}\left[ {\exp \left( { - \frac{1}{\lambda }S\left( {{{\bf{x}}_{{t_k}}}} \right)} \right)} \right]\right) - \frac{1}{\lambda }{{\mathbb{E}}_{{{\cal Q}^*}}}\left[ {S\left( {{{\bf{x}}_{{t_k}}}} \right)} \right]
\end{aligned} \label{Eq: lemma 1}
\end{equation}
Simplifying (\ref{Eq: lemma 1}) as
\begin{equation}
    \mathbb{E}_{{\cal Q}^*}\left[\log(q^*)\right] \ge \mathbb{E}_{{\cal Q}^*}\left[\log\left(\frac{{\exp \left[ { - \frac{1}{\lambda }S\left( {{{\bf{x}}_{{t_k}}}} \right)} \right]q}}{{{{\mathbb{E}}_{\cal Q}}\left[ {\exp \left( { - \frac{1}{\lambda }S\left( {{{\bf{x}}_{{t_k}}}} \right)} \right)} \right]}}\right)\right]
\end{equation}
Hence, the $\mathbb{E}_{{\cal Q}^*}$ achieves its lower bound (\ref{Eq: theorem 1}) when the optimal probability density function $q^*$ is defined as follow
\begin{equation}
{q^*} = \frac{{\exp \left[ { - \frac{1}{\lambda }S\left( {{{\bf{x}}_{{t_k}}}} \right)} \right]q}}{{{{\mathbb{E}}_{\cal Q}}\left[ {\exp \left( { - \frac{1}{\lambda }S\left( {{{\bf{x}}_{{t_k}}}} \right)} \right)} \right]}}
\end{equation}

With the definition of the optimal probability density function, (\ref{Eq: simple stochastic opt}) can be rewritten by using the standard importance sampling \cite{doucet2001sequential} trick.
\begin{equation}
{{{\bf{x}}}^*_{{t_k}}} = \int {\underbrace {\frac{{{q^*}({\bf{x}})}}{{q({\bf{x}})}}}_{w({\bf{x}})}q({\bf{x}}){\bf{x}}d{\bf{x}}}  = {{\mathbb{E}}_{\cal Q}}\left[ {w({{\bf{x}}_{{t_k}}}){{\bf{x}}_{{t_k}}}} \right] \label{Eq: q*}
\end{equation}
where $w(\cdot)$ denotes the weight of importance sampling. The expression in (\ref{Eq: q*}) computes expectation by switching sampling from $\mathcal{Q}^*$ (in (\ref{Eq: simple stochastic opt})) to sampling from $\mathcal Q$.

\section{Practical Implementation for Magnetic Field-based State Estimation} \label{sec: Practical Implementation for Magnetic Field-based State Estimation}
Equation (\ref{Eq: q*}) forms the basis of the proposed MF-based state estimation, which gives the optimal solution to the stochastic optimization problem (\ref{Eq: stochastic opt}).
Considering the expectation is hard to perfectly evaluate in practice, a Monte-Carlo sampling is adopted to approximate the optimal state derived in (\ref{Eq: q*}).
\begin{equation}
{{{\bf{\hat x}}}_{{t_k}}} = \sum\limits_{j = 0}^{M - 1} {\frac{{\exp \left[ { - \frac{1}{\lambda }S\left( {{\bf{x}}_{{t_k}}^j} \right)} \right]{\bf{x}}_{{t_k}}^j}}{{\sum\limits_{k = 0}^{M - 1} {\exp \left[ { - \frac{1}{\lambda }S\left( {{\bf{x}}_{{t_k}}^k} \right)} \right]} }}} \label{Eq: Monte-Carlo sampling}
\end{equation}
where ${{{\bf{\hat x}}}_{{t_k}}}$ denotes the estimated state, $M \in \mathbb{Z}^+$ is the number of Monte-Carlo sampling, $\mathbf{x}_{t_k}^j$ and $\mathbf{x}_{t_k}^k$ denote the random state vector obtained from $j$-th and $k$-th Monte-Carlo sampling, respectively.

\subsection{Reducing the Sampling Dimension on Manifolds}
The real-time performance of the Monte-Carlo sampling required by (\ref{Eq: Monte-Carlo sampling}) is sensitive to the sampling number $M$, the number of magnetic sensors installed on the robot $N$, and the dimension of the state vector $\mathbf{x}_{t_k} \in \mathbb{R}^{12}$.
The time complexity for evaluation of (\ref{Eq: Monte-Carlo sampling}) is $O(12^2\times MN)$, which means both Monte-Carlo sampling number $M$ and magnetic sensors number $N$ are linear dependence parameters of the time complexity, and the state dimension is even the square dependence parameters of the time complexity.
However, $M$ and $N$ are also trade-off parameters between the real-time performance and accuracy of the sampling-based state estimation methods.
Therefore, a dimension reduction sampling strategy is designed to improve the real-time performance of the proposed algorithm.
With the kinematic model, the motion of the robot at time $t_k$ can be predicted from the preceding state $\mathbf{x}_{t_{k-1}}$.
\begin{equation}
\begin{aligned}
&{{\bf{p}}_{{t_k}}} = {{\bf{p}}_{{t_{k - 1}}}} + {{\bf{v}}_{{t_{k - 1}}}}\Delta t + \frac{1}{2}{{\bf{a}}_{{t_{k - 1}}}}\Delta t^2\\
&{\bm \phi _{{t_k}}} = \text{Log}\left[ {\exp \left( {{{\left\lfloor {{{\bm{\omega }}_{{t_{k - 1}}}}\Delta t} \right\rfloor }_ \times }} \right){\mathbf{R}_{{t_{k - 1}}}}} \right]
\end{aligned} \label{Eq: kinematic model}
\end{equation}
where $\Delta t = t_k - t_{k-1}$, ${\left\lfloor \cdot \right\rfloor }_ \times$ denotes the skew-symmetric cross-product matrix, and $\text{Log}(\cdot)$ denotes the logarithmic mapping from a rotation matrix to a rotation vector. 

According to (\ref{Eq: kinematic model}), at time $t_k$, the cost function $S(\mathbf{x}_{t_k})$ can be transformed as a function of the state vector $\mathbf{\hat x}_{t_{k-1}}$ and ${\bm \tau}_{t_{k-1}} = [{\bf a}^\top_{t_{k-1}}, {\bm \omega}^\top_{t_{k-1}}]^\top \in {\mathbb R}^6$. Hence, (\ref{Eq: Monte-Carlo sampling}) can be rewritten as:
\begin{equation}
{{\bm{\hat \tau }}_{{t_{k - 1}}}} = \sum\limits_{j = 0}^{M - 1} {\frac{{\exp \left[ { - \frac{1}{\lambda }S\left( {{{{\bf{\hat x}}}_{{t_{k - 1}}}},{\bm{\tau }}_{{t_{k - 1}}}^j} \right)} \right]{\bm{\tau }}_{{t_{k - 1}}}^j}}{{\sum\limits_{k = 0}^{M - 1} {\exp \left[ { - \frac{1}{\lambda }S\left( {{{{\bf{\hat x}}}_{{t_{k - 1}}}},{\bm{\tau }}_{{t_{k - 1}}}^k} \right)} \right]} }}} \label{Eq: Reduced Monte-Carlo sampling}
\end{equation}
where ${\bm \tau}_{t_{k-1}}^j$ and ${\bm \tau}_{t_{k-1}}^k$ denote the random vector obtained from $j$-th and $k$-th Monte-Carlo sampling, respectively.
The state vector ${\mathbf{\hat x}}_{t_{k}}$ can be estimated by substituting ${\bm{\hat \tau}}_{t_{k-1}}$ and ${\mathbf{\hat x}}_{t_{k-1}}$ into the kinematic model (\ref{Eq: kinematic model}).
As a consequence, the evaluation of (\ref{Eq: Reduced Monte-Carlo sampling}) has a time complexity of $O(6^2\times MN)$ which is only a quarter of the complexity of (\ref{Eq: Monte-Carlo sampling}).

\begin{algorithm}[t]
    \caption{MF-based localization}
    \label{Alg: Gradien-free lio}
    \renewcommand{\thealgocf}{}
    \renewcommand{\theAlgoLine}{}  
    \SetKwInOut{Input}{Input}
    \SetKwInOut{Output}{Output}
    \SetKwInOut{Begin}{Begin}
    \SetKwProg{Fn}{Function}{}\\		
        \Input{$\mathbb{B}_{t_k}$: Newly obtained magnetic measurements, ${\mathbb M}$: MF map, $M$: Number of samples, $\mathbf{\hat x}_{t_{k-1}}$: The state estimated at time $t_{k-1}$.}
        \Output{$\hat{\mathbf{x}}_{t_k}$: The estimated state vector.}
    
    \For{$j\gets 0 \,\, \textbf{to} \,\, M-1$}
    {
        $\bm{\tau}^j_{t_{k-1}} \gets Sampling$\\
        Get ${\bf{x}}^j_{t_k}$ by substitute ${{\bm{\tau }}^j_{{t_{k - 1}}}}$ and ${\bf{\hat x}}_{t_{k-1}}$ into the kinematic model (\ref{Eq: kinematic model})\\
        \For {each $\mathbf{B}^i_{t_k} \in \mathbb{B}_{t_k}$} 
        {
            Calculate the pose of $i$-th magnetic sensor over $j$-th sampling as: ${\bf{R}}_{{t_k}}^{ij} = {\bf{R}}_{{t_k}}^j {^b{\bf{R}}_m^i}, \quad{\bf{p}}_{{t_k}}^{ij} = {\bf{R}}_{{t_k}}^j {^b{\bf{p}}_m^i} + {\bf{p}}_{{t_k}}^j$\\
            \If{$IsOutlier(\mathbf{R}^{ij}_{t_k},\mathbf{p}^{ij}_{t_k})$ is \text{True}}
            {
                $S(\mathbf{x}^{j}_{t_k}) = S(\mathbf{x}^{j}_{t_k}) + c^2$
            }
            \Else
            {
                $S(\mathbf{x}^{j}_{t_k}) = S(\mathbf{x}^{j}_{t_k}) + {{{\left\| {{\mathbb{M}}({\bf{p}}_{{t_k}}^{ij}) - {\bf{R}}_{{t_k}}^{ij}{\bf{B}}_{{t_k}}^i} \right\|}^2}}$
            }
        }
    }
    $S_{\min} = \min\left[S(\mathbf{x}_{t_k}^j)\right], j=0,\cdots,M-1$\\
    ${{\bm{\hat \tau }}_{{t_{k - 1}}}} = \sum\limits_{j = 0}^{M - 1} {\frac{{\exp \left[ { - \frac{1}{\lambda }\left(S\left( {{{{\bf{\hat x}}}_{{t_{k - 1}}}},{\bm{\tau }}_{{t_{k - 1}}}^j} \right)-S_{\min}\right)} \right]{\bm{\tau }}_{{t_{k - 1}}}^j}}{{\sum\limits_{k = 0}^{M - 1} {\exp \left[ { - \frac{1}{\lambda }\left(S\left( {{{{\bf{\hat x}}}_{{t_{k - 1}}}},{\bm{\tau }}_{{t_{k - 1}}}^k} \right)-S_{\min}\right)} \right]} }}}$ \\
    Get ${\bf{\hat x}}_{t_k}$ by substitute ${{\bm{\hat \tau }}_{{t_{k - 1}}}}$ and ${\bf{\hat x}}_{t_{k-1}}$ into the kinematic model (\ref{Eq: kinematic model}).\\
    \Return {${\bf{\hat x}}_{t_k}$}
\end{algorithm}
\subsection{Improve the Numerical Stability: Cost Shifting}
Consider the negative exponentiation required by (\ref{Eq: Reduced Monte-Carlo sampling}) is numerically sensitive to the range of the input values \cite{8558663}.
The range of MF matching cost should be shifted to make the cost of the best Monte-Carlo sample have a value of zero. Define $S_{\min}$ as the minimum cost of $S\left( {{{{\bf{\hat x}}}_{{t_{k - 1}}}},{\bm{\tau }}_{{t_{k - 1}}}^j} \right), j=0,\cdots,M-1$, and multiply $\exp(\frac{1}{\lambda}S_{\min})/\exp(\frac{1}{\lambda}S_{\min})$ by (\ref{Eq: Reduced Monte-Carlo sampling}) results in
\begin{equation}
{{\bm{\hat \tau }}_{{t_{k - 1}}}} = \sum\limits_{j = 0}^{M - 1} {\frac{{\exp \left[ { - \frac{1}{\lambda }\left(S\left( {{{{\bf{\hat x}}}_{{t_{k - 1}}}},{\bm{\tau }}_{{t_{k - 1}}}^j} \right)-S_{\min}\right)} \right]{\bm{\tau }}_{{t_{k - 1}}}^j}}{{\sum\limits_{k = 0}^{M - 1} {\exp \left[ { - \frac{1}{\lambda }\left(S\left( {{{{\bf{\hat x}}}_{{t_{k - 1}}}},{\bm{\tau }}_{{t_{k - 1}}}^k} \right)-S_{\min}\right)} \right]} }}} 
\end{equation}
which can prevent numerical overflow or underflow and improve the stability of the proposed MF-based localization. The complete algorithm is summarized in Algorithm \ref{Alg: Gradien-free lio}.
\section{Results}
In this section, the effectiveness of the proposed MF-based localization method is demonstrated through a series of experiments. To quantitatively verify the capability of the proposed method, a high-fidelity physical simulation environment is deployed with absolute ground truth. Moreover, real-world experiments are conducted with an autonomous mobile robot to further demonstrate the practical applicability and robustness of the proposed method.
\begin{itemize}
    \item High-Fidelity Physical Simulation: The Gazebo software package\footnote{\url{https://gazebosim.org/home}} is employed to resemble the actual industrial warehouse. As shown in \mbox{Fig. \ref{Simulation Scenario}}, the simulated industrial warehouse consists of repetitive goods shelves, steel concrete pillars, and conveyor belts. The local MF variations are modeled according to Gaussian process regression \cite{8373720} and overlaid on the basic geomagnetic field as shown in \mbox{Fig. \ref{MF distribution}}.
    \item Real-world Experiments: A Clearpath Husky A200 autonomous mobile robot platform\footnote{\url{https://clearpathrobotics.com}} is deployed to collect the MF data in various real-world environments, including office corridor, semi-indoor carpark, and office building. As shown in \mbox{Fig. \ref{Husky}}, the Husky A200 platform is equipped with two magnetic sensor arrays (each array consists of 7 magnetic sensors, and only the top array is used for online localization), a VN-100T IMU, an Ouster OS1-32 LiDAR, and a laptop with an Intel i7-10875H CPU @2.3GHz and 32GB RAM. 
    Ground truth for all the real-world data sequences is generated by combine the state-of-the-art LiDAR-inertial odometry FAST-LIO2\cite{FASTLIO2} with loop closure\cite{8593953} and global BA\cite{10263983}.
\end{itemize} 
\begin{figure}
\subfigure[Simulated Warehouse]{
    \includegraphics[width=3.55cm]{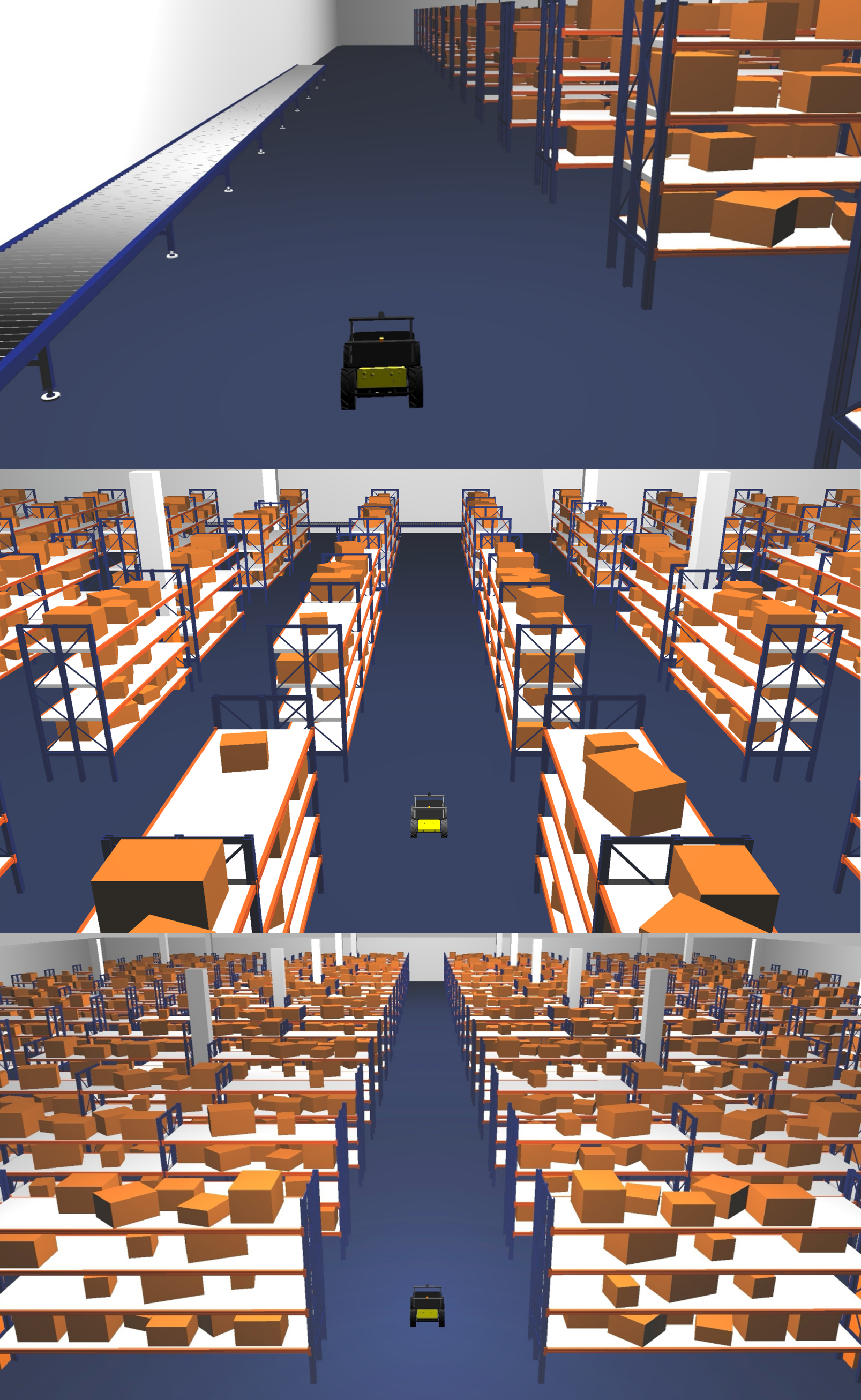}
    \label{Simulation Scenario}		
}
\subfigure[Husky A200 mobile platform]{
    \includegraphics[width=4.65cm]{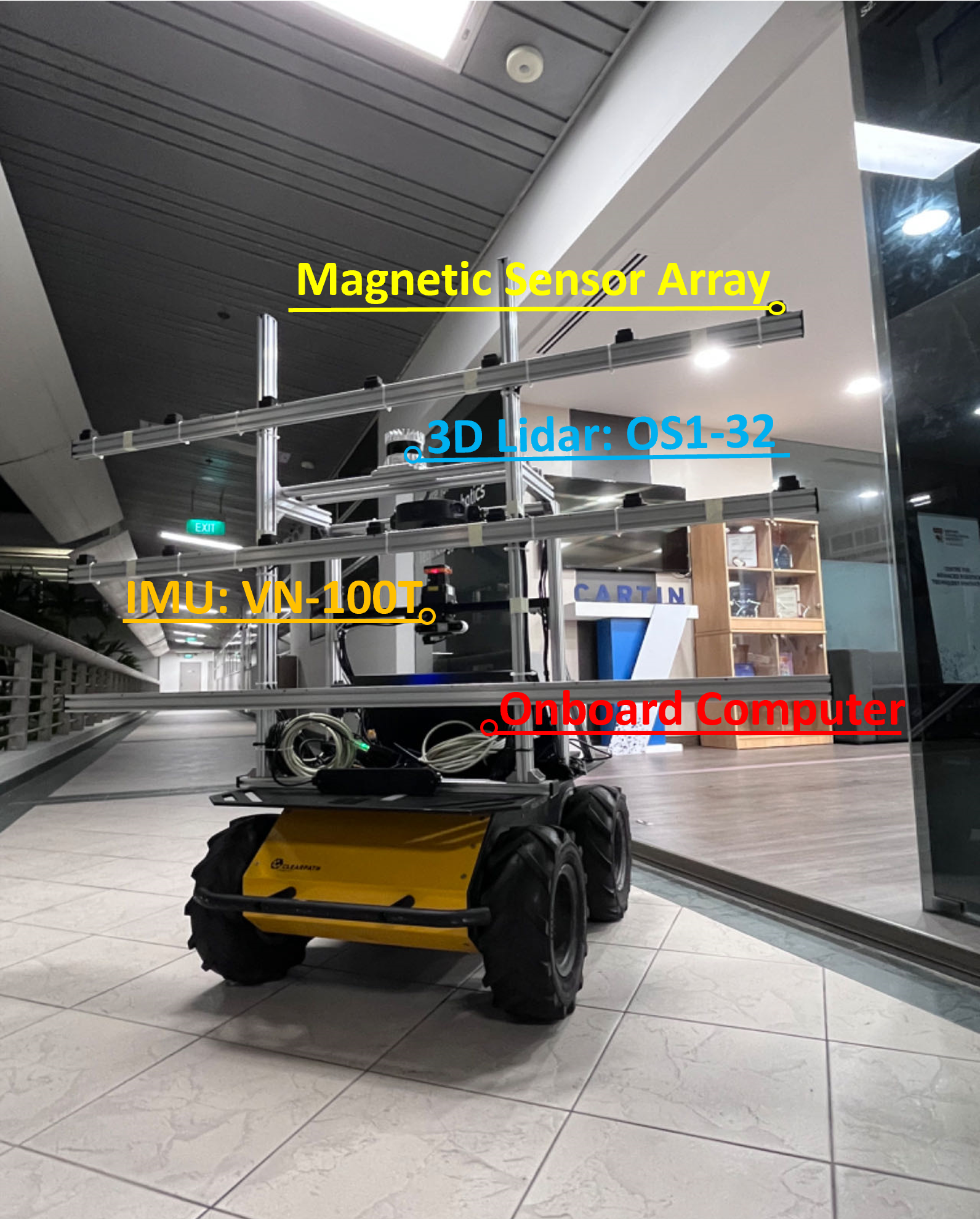} 
    \label{Husky}
}
\vspace{-1em}
\caption{High-fidelity physical simulation environment and the mobile robot platform used in real-world experiments.} \label{Kitti comparision}
\vspace{-1em}
\end{figure}
\begin{figure*}[!t]\centering
\includegraphics[width=\linewidth]{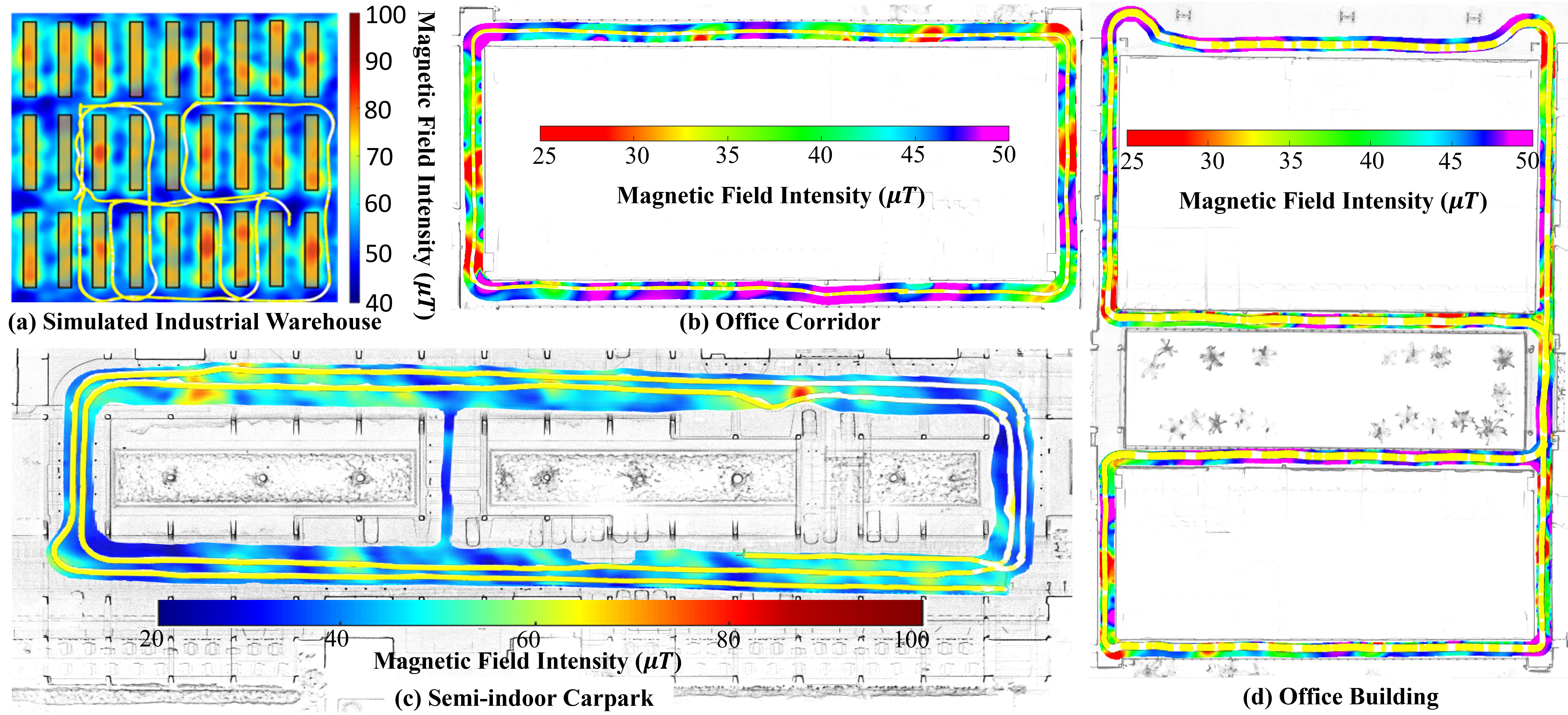}
\vspace{-1em}
\caption{MF distribution and trajectory estimation results. In each map, color designates the MF integrated intensity ($l^2-$norm) according to the scale shown in each map. Trajectories estimated by the proposed method and ground truth are noted with white and yellow, respectively.}\label{MF distribution}
\vspace{-1em}
\end{figure*}
\subsection{Comparison Baseline} \label{Sec: Comparison baseline}
To illustrate the effectiveness of the proposed method, we present detailed quantitative analyses of the proposed with state-of-the-art MF-based localization methods, which include 
1) \textit{Probability-based}\cite{WUICRA23}: a pure MF localization method realized through the MAP estimation; 2) \textit{Gradient-based}\cite{MFGN}: a PDR-aided MF localization method utilizes Guass-Newton iterative and EKF; 3) \textit{PF-based}\cite{MFPF}: a PDR/MF fusion localization method based on PF; 4) \textit{DTW-based}\cite{9947016}: a PDR/MF fusion localization method based on DTW.
\begin{rem} \label{rem: wheel odom aid}
Most MF-based localization methods rely on foot-mounted sensors and PDR algorithms, such as \cite{MFGN,MFPF,9947016}, which cannot provide state estimation for non-legged robots. 
For a fair comparison, the PDR algorithm output used in \cite{MFGN,MFPF,9947016} is replaced by the wheel odometry information in both simulation and real-world experiments.
\end{rem}
\subsection{Accuracy Evaluations} \label{sec: Accuracy Evaluations}
The absolute trajectory error (ATE) results of each method listed in \mbox{Section \ref{Sec: Comparison baseline}} are shown in \mbox{Table \ref{tab: ATE}}. 
From the results, the proposed method provides high localization accuracy in all data sequences, with an average ATE of $0.0852m$.
The proposed method achieves an accuracy improvement of over $52\%$ and $23\%$ when compared with the state-of-the-art pure MF localization method \cite{WUICRA23} and the wheel odometry-aided MF localization method \cite{MFGN,MFPF,9947016}, respectively. 
The performance improvement mainly attributed to the outlier measurements is not neglected by the proposed method in \mbox{Section \ref{sec: Problem formulation}}.
Generally, MF localization methods \cite{WUICRA23,MFGN,MFPF,9947016} formulate the state estimation problem as a MAP problem and solve the state approximately through the Gaussian distribution assumption.
The proposed method approximates the lower bound of the optimal expected cost function (\textit{Theorem \ref{Them: 1}}) through Monte-Carlo sampling, which can better handle the non-Gaussian noise introduced by outlier measurements.
Especially for the Carpark environments, the presence of dynamic vehicles introduces a higher number of outlier observations compared to the corridor, office building, and simulation environments.
Hence, compared with the state-of-the-art MF-based localization methods, the proposed method achieves a $45\%$ to $77\%$ accuracy improvement in a carpark environment.
Comparison between the proposed method and ground truth over all data sequences are shown in \mbox{Fig. \ref{MF distribution}}, which shows the proposed method achieves consistent trajectory estimation with the ground truth.
\begin{table}[!t] \centering
    \centering
	\caption{The Comparison of Absolute Trajectory Error (ATE, Meters). The best result is highlighted in {\color{black}\bf{BOLD}}.}
        \vspace{-1em}
	\label{tab: ATE}
	\begin{threeparttable}
        \begin{tabular}{l c c c c c}
        \hline\hline
    	Method    					&{Proposed} 	&{Probability} 	&{Gradient} 	&{PF} 	&{DTW} 	\\\hline
            Warehouse                   &0.1204         &0.1566         &0.1001         &\bf0.0957 &1.2716\\
    	Corridor             &\bf0.0582         &0.1561         &0.0933         &0.1062 &0.1119\\
    	Building             &\bf0.0865         &0.1581         &0.0873         &0.0947 &0.1407\\
    	Carpark                     &\bf0.0756         &0.2541         &0.3404         &0.1492 &0.2092\\\hline
    
    	Average                     &\bf0.0852         &0.1812         &0.1553         &0.1114 &0.4334\\
    	\hline\hline
        \end{tabular}
	\end{threeparttable}
 \vspace{-2em}
\end{table}
\subsection{Robustness Evaluations}\label{sec: Robustness Evaluations}
The robustness of the MF-based localization methods is largely based on avoiding degeneracy and local minimum that arise from cases such as scarcity of MF gradient information and inevitable mismatchings.
As described in \textit{Remark \ref{rem: wheel odom aid}}, the wheel odometry information is adopted by both Gradient-based\cite{MFGN}, PF-based\cite{MFPF}, and DTW-based\cite{9947016} methods in case degeneracy occurs. For a fair comparison, a MF-based global localization method\cite{10370015,wu2024mglt} is introduced to re-initialize the MF-based localization methods from failed cases such as the gradient-based optimizer cannot converge, PF degeneracy, and all candidates of DTW or MAP estimator (used in the Probability-based method\cite{WUICRA23}) have similar cost.
Comparison results of the localizability over the proposed method with state-of-the-art methods are shown in \mbox{Fig. \ref{Robust compara}}.
From the result, both \cite{WUICRA23,MFGN,MFPF,9947016} are not able to solve the MF-based localization problem reliably.
Three factors contribute to this: 1) degeneracy caused by insufficient MF variation (especially for the gradient-based method). 2) inaccuracies in the correspondence search, affected by the bad initial guess. 3) measurement outliers introduced by external disturbance, such as vehicles, pedestrian-carried smartphones.
Thanks to the penalize of outliers in (\ref{Eq: TLS}) and the robust optimal solver ((\ref{Eq: q*}) in an information-theoretic sense) designed in Section \ref{sec: Robust Estimation via Stochastic Optimization}, the proposed method produces a reliable state estimation in both date sequences, as shown in \mbox{Fig. \ref{Robust compara}}.

\begin{figure*}[!t]\centering
\includegraphics[width=\linewidth]{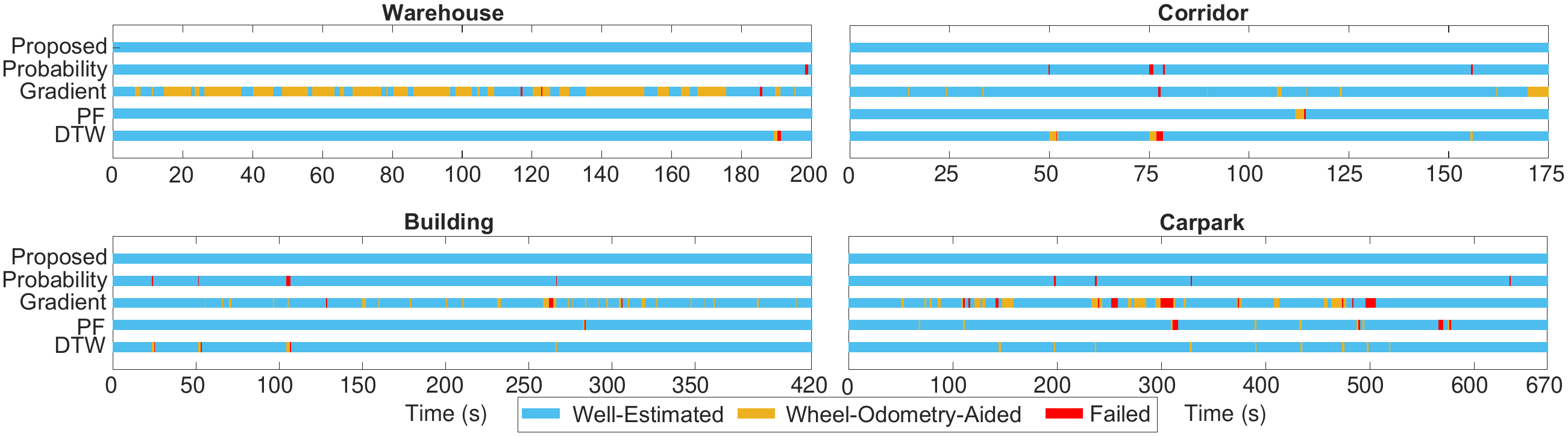}
\vspace{-1em}
\caption{Comparations of localizability for investigated MF-based localization methods.}\label{Robust compara}
\vspace{-1em}
\end{figure*}
\subsection{Running Time Evaluations}
The localization efficiency is summarized in \mbox{Table \ref{tab: realtime}}. 
From the results, the average state estimation time of the proposed method is well-bounded within $10 ms$, which allows MF localization to take place in real-time using the high-frequency magnetic sensor at a rate of $100 Hz$.
Although the gradient-based MF localization method achieves significant real-time performance, as described in Section \ref{sec: Accuracy Evaluations} and Section \ref{sec: Robustness Evaluations}, it cannot realize reliable MF-based localization.
Thanks to the fully paralleled stochastic optimization algorithm and the sampling dimension-reducing strategy introduced in Section \ref{sec: Robust Estimation via Stochastic Optimization} and Section \ref{sec: Practical Implementation for Magnetic Field-based State Estimation}, compared with the pure MF-based localization method\cite{WUICRA23}, the proposed method achieves a significant efficiency improvement, which saves $61\%$ to $78\%$ processing time.
Furthermore, the proposed method provides a comparable real-time performance with the PF-based method \cite{MFPF} while achieving a high-accuracy velocity estimation (as shown in \mbox{Fig. \ref{vel_est}}).
These results demonstrate the advantage of the proposed method over previous works that only considered pose estimation \cite{WUICRA23,MFGN,MFPF,9947016}.
As shown in \mbox{Fig. \ref{vel_est}}, the proposed method achieves smoother velocity estimation when compared with the ground truth provided by the state-of-the-art LiDAR-inertial odometry FAST-LIO2\cite{FASTLIO2}.

\begin{table}[!t] \centering
    \centering
	\caption{The Comparison of Average Processing Time (Milliseconds). The best result is highlighted in {\color{black}\bf{BOLD}}.}
	\label{tab: realtime}
	\begin{threeparttable}
        \begin{tabular}{l c c c c c}
        \hline\hline
    	Method    					&{Proposed} 	&{Probability} 	&{Gradient} 	&{PF} 	&{DTW} 	\\\hline
            Warehouse                   &1.2834         &5.8672         &\bf0.2410         &1.4257 &14.0315\\
    	Corridor             &4.5387         &17.3779         &\bf0.6764         &4.8696 &43.5775\\
    	Building             &9.2516         &32.3832         &\bf1.0178         &7.9291 &72.8186\\
    	Carpark                     &9.7129         &24.7835         &\bf0.8594         &7.1066 &54.3153\\\hline

    	\hline\hline
        \end{tabular}
	\end{threeparttable}
\end{table}
\begin{figure}[!t]\centering
\includegraphics[width=\linewidth]{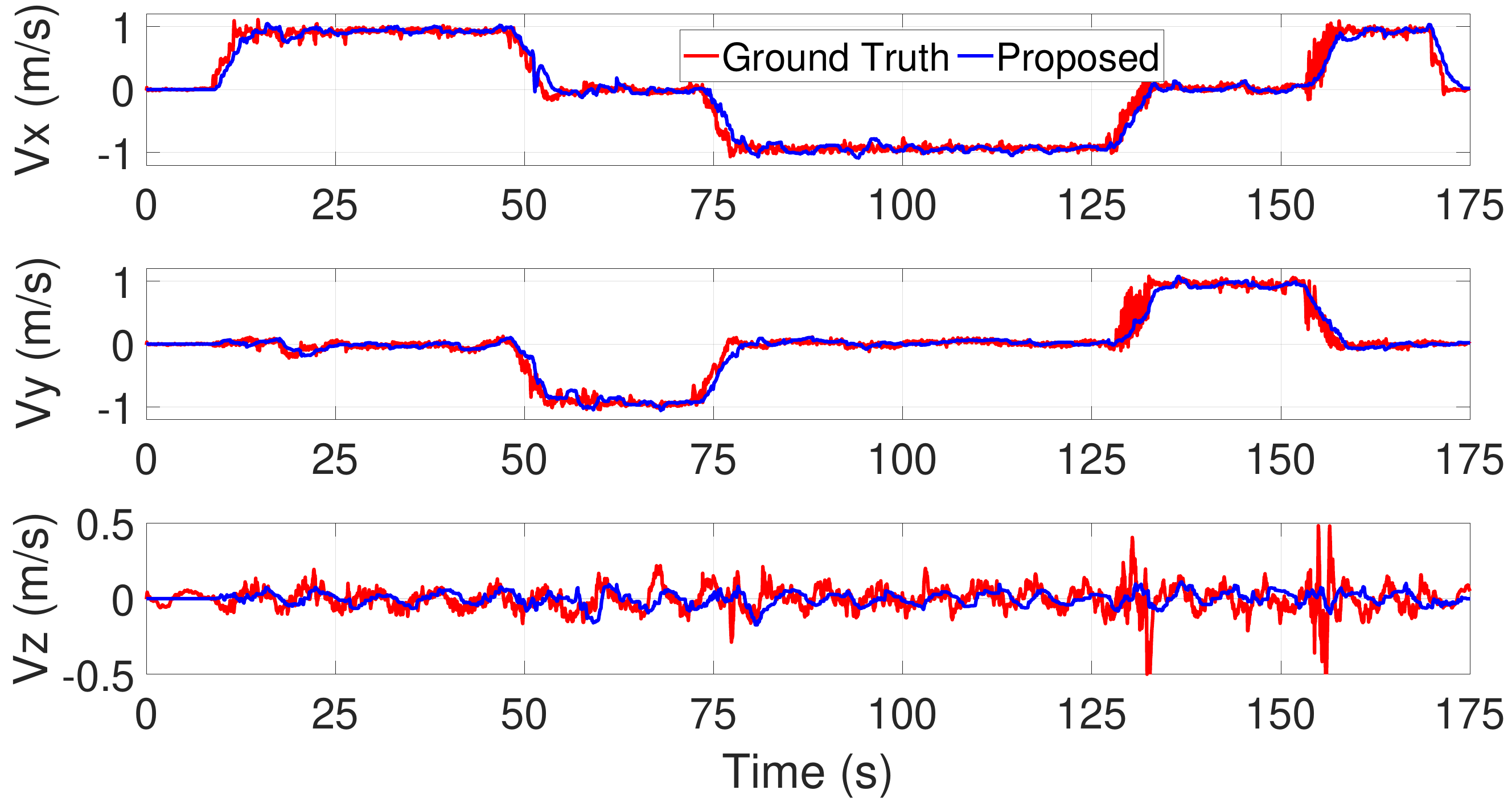}
\vspace{-2em}
\caption{Velocity estimated from the proposed method and ground truth under the Corridor sequence. It is worth noting that velocity can change drastically when loop closure occurs. The velocity ground truth is generated by FAST-LIO2\cite{FASTLIO2}.}\label{vel_est}
\vspace{-1em}
\end{figure}
\section{Conclusion}
In this article, an infrastructure-free and drift-free localization system, IDF-MFL, is developed for mobile robots using the ambient MF information.
The proposed method formulates the MF-based localization problem as a MF matching problem in a stochastic optimization formulation, which defines the MF matching cost in a piecewise function formulation with the consideration of outliers.
The optimal solution of the MF-based localization problem is derived by finding the optimal distribution of the robot state that makes the expectation of MF matching cost achieve its lower bound, and a detailed mathematical lower bound for the cost function expectation is carefully derived.
To deploy the proposed stochastic estimator in practice, a series of practical implementations are given in detail, including Monte-Carlo approximation, sampling dimension shrinking, and MF matching cost-shifting strategy.
With the successful implementation of the proposed method using C++ and ROS framework, extensive experiments are conducted to verify the practicability and performance of the proposed MF-based localization algorithm. 
The results show that IDF-MFL produces high-accuracy, reliable, and real-time localization results without any pre-installed infrastructures, and has a great potential to enable mobile robot applications.

\bibliographystyle{Bibliography/IEEEtran}
\bibliography{Bibliography/IEEEabrv,Bibliography/IEEEexample}
\end{document}